\documentclass[sigconf]{acmart}

\AtBeginDocument{%
  \providecommand\BibTeX{{%
    \normalfont B\kern-0.5em{\scshape i\kern-0.25em b}\kern-0.8em\TeX}}}

\usepackage{algorithm}
\usepackage{algorithmic}
\usepackage{multirow}

\newcommand{\argmin}{\mathop{\rm arg~min}\limits}

\setcopyright{acmcopyright}
\copyrightyear{2020}
\acmYear{2020}
\setcopyright{iw3c2w3}
\acmConference[WWW '20]{Proceedings of The Web Conference 2020}{April 20--24, 2020}{Taipei, Taiwan}
\acmBooktitle{Proceedings of The Web Conference 2020 (WWW '20), April 20--24, 2020, Taipei, Taiwan}
\acmPrice{}
\acmDOI{10.1145/3366423.3380032}
\acmISBN{978-1-4503-7023-3/20/04}

\begin{document}

\title{A Feedback Shift Correction in Predicting Conversion Rates under Delayed Feedback}


\author{Shota Yasui}
\authornote{Both authors contributed equally to this research.}
\email{yasui\_shota@cyberagent.co.jp}
\affiliation{%
  \institution{Cyberagent, Inc.}
  \streetaddress{Udagawacho, 40−1 Abema Towers}
  \city{Shibuya}
  \state{Tokyo}
  \postcode{150-0042}
}

\author{Gota Morishita}
\authornotemark[1]
\email{morishit\_gota@cyberagent.co.jp}
\affiliation{%
  \institution{Cyberagent, Inc.}
  \streetaddress{Udagawacho, 40−1 Abema Towers}
  \city{Shibuya}
  \state{Tokyo}
  \postcode{150-0042}
} 

\author{Komei Fujita}
\email{fujita\_komei@cyberagent.co.jp}
\affiliation{%
  \institution{Cyberagent, Inc.}
  \streetaddress{Udagawacho, 40−1 Abema Towers}
  \city{Shibuya}
  \state{Tokyo}
  \postcode{150-0042}
} 

\author{Masashi Shibata}
\email{shibata\_masashi@cyberagent.co.jp}
\affiliation{%
  \institution{Cyberagent, Inc.}
  \streetaddress{Udagawacho, 40−1 Abema Towers}
  \city{Shibuya}
  \state{Tokyo}
  \postcode{150-0042}
}

%

\begin{abstract}
In display advertising, predicting the conversion rate, that is, the probability that a user takes a predefined action on an advertiser's website, such as purchasing goods is fundamental in estimating the value of displaying the advertisement.
However, there is a relatively long time delay between a click and its resultant conversion. Because of the delayed feedback, some positive instances at the training period are labeled as negative because some conversions have not yet occurred when training data are gathered. 
As a result, the conditional label distributions differ between the training data and the production environment.
This situation is referred to as a \textit{feedback shift}. 
We address this problem by using an importance weight approach typically used for covariate shift correction. 
We prove its consistency for the feedback shift. 
Results in both offline and online experiments show that our proposed method outperforms the existing method.

\end{abstract}

\ccsdesc[500]{Information systems~Display advertising}
\ccsdesc[500]{Computing methodologies~Learning under covariate shift}

\keywords{Delayed Feedback, Conversion Prediction, Importance Weight}

\maketitle

\section{Introduction}
Over the last decade, programmatic advertisement (ad) buying through real-time auction has become common in performance display advertising. 
Advertisers have been offered several payment options, such as paying per impression (CPM), paying per click (CPC), and paying per conversion (CPA). 
The CPA option is preferred by advertisers because they would rather pay for a conversion which is more likely to lead to profits. 
Thus, we focus on a CPA model in which advertisers pay only if a user performs a predefined action on their website after clicking on the advertisement. 
In this payment model, accurately predicting a conversion rate (CVR) is essential in estimating the value of an ad impression.
However, there is a delay between a click and its resultant conversion.
It takes some time for a conversion to occur following a click.
In the production environment, the training data are collected right before training a model. 
Therefore, some conversions are not observed yet for samples observed near the training timing.
This leads to mislabeling some samples in training data. 
Consequently, there is a discrepancy between the conditional label distribution of the training data and that of the test data because the test data are tracked for a sufficiently long period to ensure accurate labeling.
It is possible to wait for a fixed time window before assigning a label to ad clicks and then train on the data. 
However, as discussed in \cite{chapelle2014modeling}, a shorter window increases the likelihood of positive training samples being mislabeled as negative.
Training samples are, at the least, as old as a window length; thus a longer window tends to generate a stalled model because of a large shift over time resulting from factors, such as seasonality and changes to ad campaigns.

Work has already been completed regarding the delayed feedback issue. 
To the best of our knowledge, this issue was first addressed in \cite{chapelle2014modeling}. \cite{chapelle2014modeling} assumes that the delay distribution is exponential and proposes two models: one predicts the CVR and the other predicts the delay in conversion. 
Both models are jointly trained via the expectation-maximization (EM) algorithm or the gradient descent optimization.
\cite{imai2018NoDeF} extends this approach and proposes use of a non-parametric model for delay distribution estimation.

Whereas the previous studies focus on the CVR prediction task, \cite{Ktena2019AddressingDF} examines the click-through-rate (CTR) prediction for a video ad. 
On their platform, there is a severe feature distribution shift that necessitates training a model online on fresh data, giving rise to the delayed feedback problem.

In this study, we regard the delayed feedback as a data shift in which there is a disparity between the label distribution in the training and test data although the feature distributions remain the same. 
We term this situation a \textit{feedback shift}. 
This concept is closely related to the covariate shift, wherein there is a disparity between the feature distribution in the training and test data whereas the conditional label distributions remain the same\cite{Shimodaira2000ImprovingPI}. 
Similar to the covariate shift, the feedback shift also results in an inconsistent empirical loss function that degrades the performance of a model. 
To address the feedback shift, we propose the importance weighting (IW) approach, which is a well-known solution to the covariate shift \cite{Shimodaira2000ImprovingPI}.
As with \cite{sugiyama2007iwcv}, the importance weighted estimate of the loss under \textit{feedback shift} is also consistent, as is shown subsequently. 
However, because the IW requires the test data distribution which is unavailable, we have to estimate the IW, and then train a CVR model using the estimated IW.

We conducted two offline and one online experiments. For offline experiments, we used different datasets: a public dataset, the conversion logs dataset provided by Criteo\footnote{https://labs.criteo.com/2013/12/conversion-logs-dataset/} and an in-house dataset provided by Dynalyst\footnote{http://www.dynalyst.io/}.
Using the public data, we conducted the first experiment to demonstrate the effectiveness of our proposed method.
In the second experiment, we used the offline in-house data, and incorporated the IW approach into the field-aware factorization machines (FFM) \cite{Juan2016FieldawareFM}; subsequently we evaluated the derived method, hereafter referred to as FFMIW, on the offline in-house data, and demonstrated its superiority to the FFM under a specific circumstance.
Finally, based on the offline result, we decided to conduct an online A/B test to confirm the effectiveness of our proposed method in the production system.

\section{Related Work}
Some works on CVR prediction ignore the delayed feedback problem\cite{Agarwal2010EstimatingRO, Menon2011ResponsePU, Lee2012EstimatingCR, Rosales2012PostclickCM, Juan2017FieldawareFM}. 
Following the publication of \cite{chapelle2014modeling}, the delayed feedback has been a focus of constant attention. 
Some recent studies have attempted to solve it \cite{Tallis2018ReactingTV, Ktena2019AddressingDF, imai2018NoDeF, Safari2017DisplayAE}.
Furthermore, \cite{chapelle2014modeling} states that the delayed feedback problem is related to positive-unlabeled(PU) Learning.

We considered the delayed feedback problem as a feedback shift. 
The feedback shift can be addressed using an IW approach, similar to the covariate shift correction. 
Covariate shift, also known as sample selection bias, has been extensively studied \cite{Zadrozny2004LearningAE, Huang2006CorrectingSS, Sugiyama2007DirectIE,Shimodaira2000ImprovingPI, Gretton20081CS}.
Another similar concept is termed the label shift, wherein there is a disparity between test and training label distributions ,which is a general case of feedback shift\cite{Lipton2018DetectingAC, Storkey2013WhenTA, Zhang2013DomainAU}.

The delayed feedback in the bandit algorithm has been researched\cite{Joulani2013OnlineLU, Vernade2017StochasticBM, Vernade2018ContextualBU, PikeBurke2017BanditsWD}. 
Whereas the objective in the bandit problem is to sequentially make decisions in order to minimize the cumulative regret, our goal is to predict the CVR in order to derive a bid price in ad auction.
It is also necessary to consider the delay of the reward feedback in the bandit problem. \cite{Joulani2013OnlineLU, Vernade2017StochasticBM, Vernade2018ContextualBU} updates the reward estimator just at the moment that the reward is observed.

\section{Pre-analysis}
We analyze the Criteo and Dynalyst datasets to show that the delay exists in reality.

First, we calculated the CVR at different time intervals in the Criteo dataset. 
Figure \ref{fig:criteo_delay} shows that there is a delay between a click and its resultant conversion. While 30\% of the conversions occur less than 1 hour after the click, one-half of them occur after one day. Additionally, 13\% of them occur after two weeks. These long delays imply that a large portion of the samples in the training data are incorrectly labeled; thus, the label distributions in the training and test data are different.

As discussed in Section 1, one of the ways to circumvent mislabeling arising from the delay is to wait for a fixed time window before assigning a label to the samples. 
However, as shown in Figure \ref{fig:criteo_delay}, to correctly label approximately 90\% of samples, at least two weeks is required. 
If the fixed time window is two weeks, samples that are collected less than two weeks before would be unlabeled. Therefore, they would not be used to train the model to predict the CVR. However, this causes a data shift between the training and test data because of the frequent changes to ad campaigns. For example, 11.3\% of the traffic came from new campaigns \cite{chapelle2014modeling}. 
In addition to the changes to ad campaigns, the arrival of new products and special events affect the CVR.
For example, in the Dynalyst dataset, an upward trend is observed during a specific period in Figure \ref{fig:dyna_game_event}.
Therefore, merely waiting for a fixed time window is not an optimal way to eliminate the effect of delayed feedback in CVR prediction.
Although we can avoid the data shift by using fresh data, the fresh data are likely to be mislabeled due to the delayed feedback as discussed in Introduction.
It is key to use fresh data and eliminate the effect of the delay.

\begin{figure}[h]
\includegraphics[width=6cm, height=3.3cm]{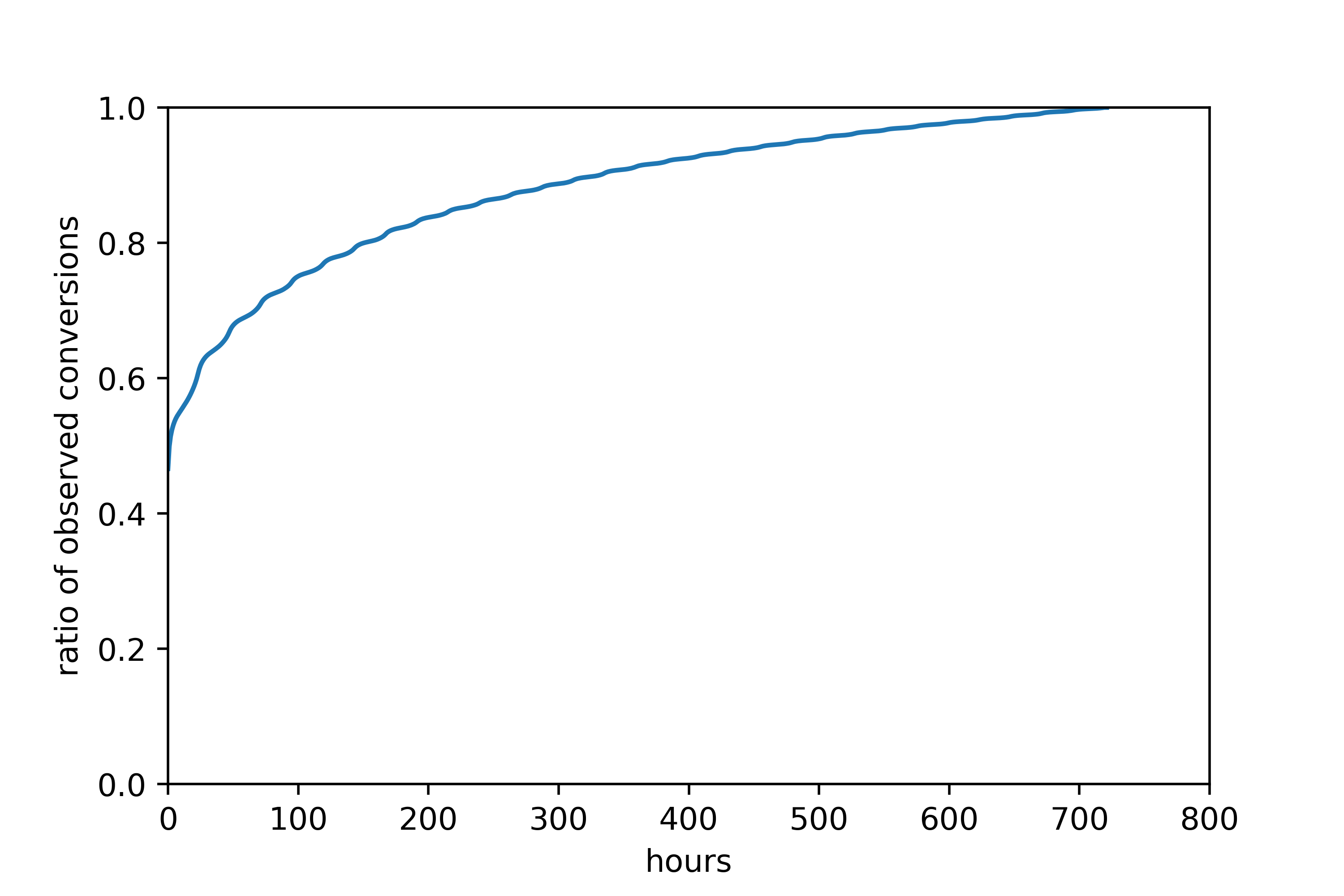}
\caption{Criteo Dataset: Cumulative distribution of the delay between the click and its conversion.}
\label{fig:criteo_delay}
\end{figure}

\begin{figure}[h]
\includegraphics[width=6cm, height=3.3cm]{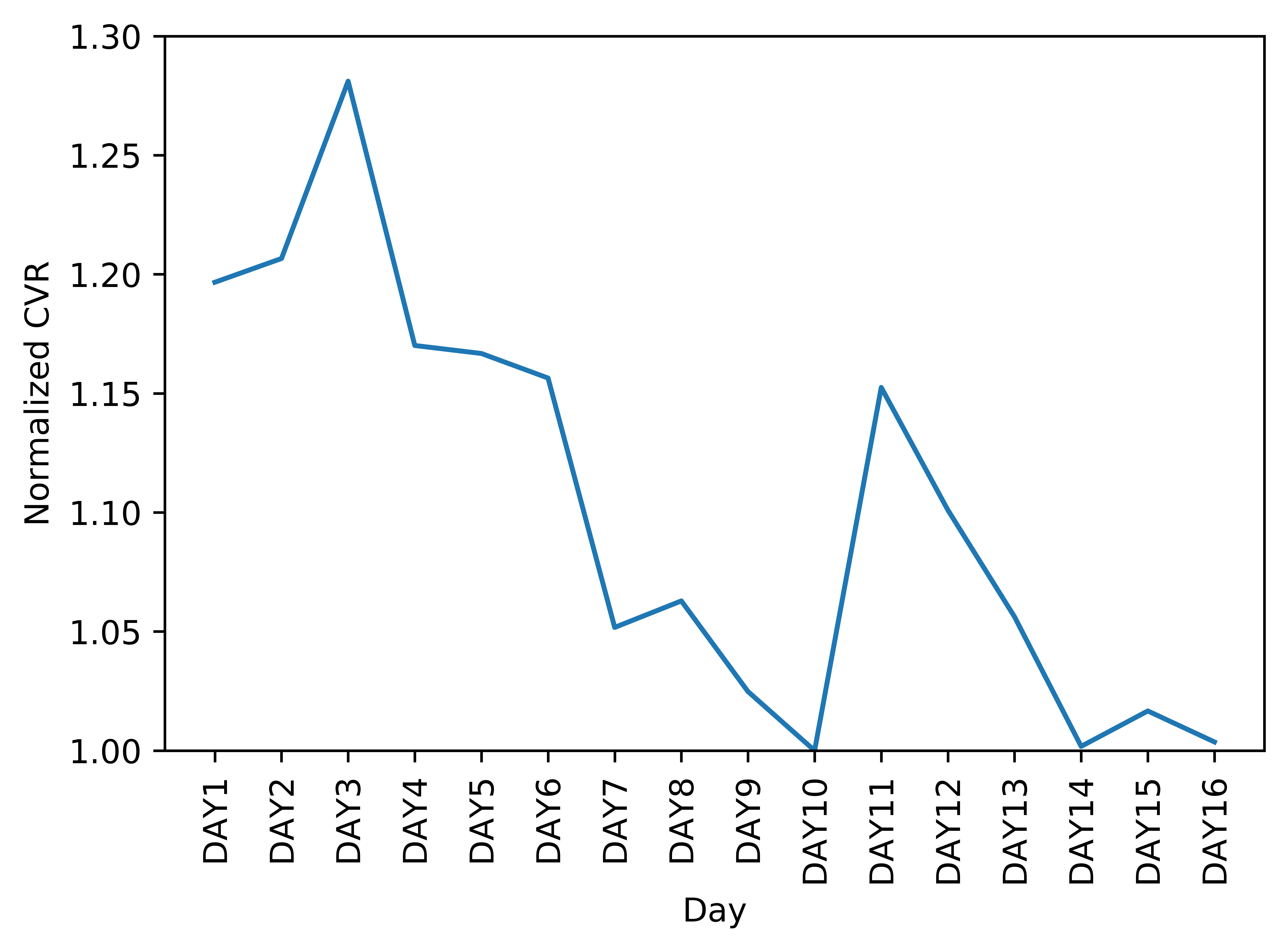}
\caption{Dynalyst Dataset: Normalized observed CVR.}
\label{fig:dyna_game_event}
\end{figure}

Secondly, we calculated the empirical probability density functions of the delay in the Criteo and Dynalyst datasets. 
There is a 24-hour cyclicality in the Criteo dataset because people browse and surf the internet more at a certain times of the day\cite{chapelle2014modeling}. 
The cyclicality patterns of different campaigns are varied, as shown in Figure \ref{fig:criteo_camp_delay}, \ref{fig:dyna_adv_delay}. 
\cite{chapelle2014modeling} ignores the periodicity and assumes the delay follows the exponential distributions. However, our proposed method using the IW can capture the oscillating shapes of a delay distribution if a flexible model is used to estimate the IW. 
This leads to improved CVR prediction.

\begin{figure}[h]
\includegraphics[width=8cm, height=2.8cm]{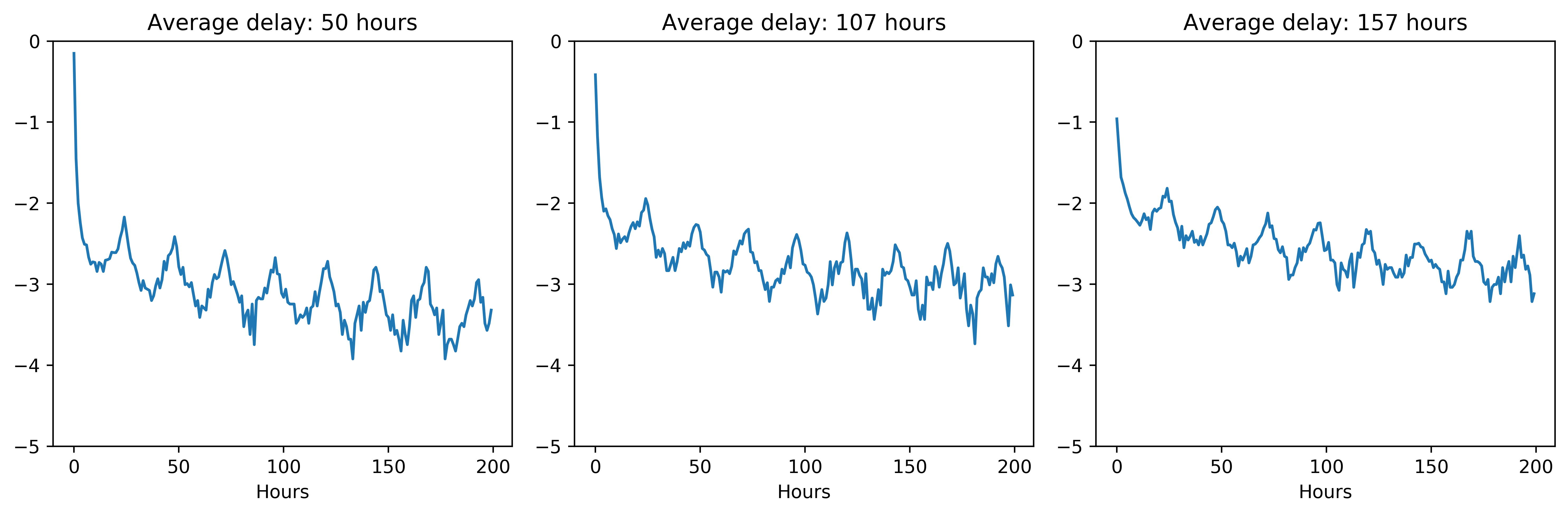}
\caption{Criteo Dataset: Probability density functions of the delays between clicks and conversions for three different campaigns chosen among many campaigns that have different average delays.}.
\label{fig:criteo_camp_delay}
\end{figure}    

\begin{figure}[h]
\includegraphics[width=8cm, height=2.8cm]{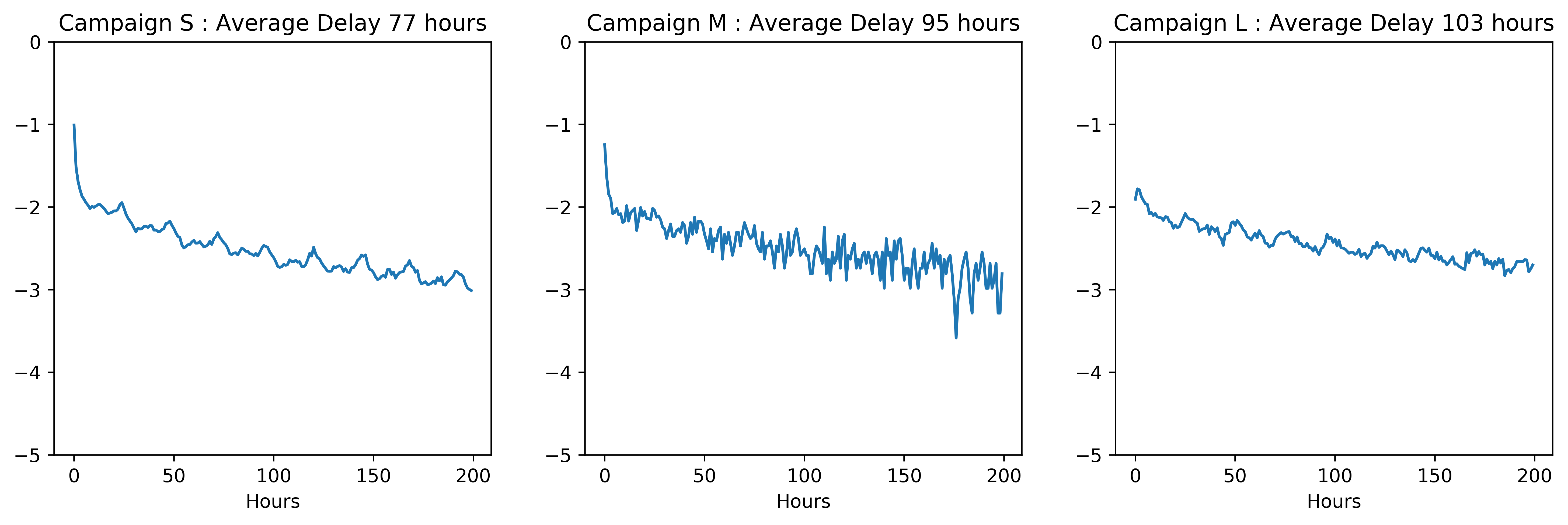}
\caption{Dyanlyst dataset: Probability density functions of the delays between clicks and conversions for three selected campaigns with different average delays.}
\label{fig:dyna_adv_delay}
\end{figure}

\section{Delayed Feedback}
The structure of the delayed feedback in the CVR prediction is described in this section. 
We show that we can see this problem as a feedback shift where the training and test conditional label distributions are different.

\subsection{Delayed Feedback Formulation}
We define some random variables as follows.
\begin{itemize}
    \item $X$:  $\mathcal{X}$-valued random variable of features;
    \item $Y$:  $\{0, 1\}$-valued random variable indicating whether a conversion occurs during the training term if $Y=1$;
    \item $C$:  $\{0, 1\}$-valued random variable indicating whether a conversion occurs if $C=1$. 
    \item $S$: $\{0, 1\}$-valued random variable indicating whether a sample is correctly labeled if $S=1$ in the training data. 
    \item $D$: $\mathbb{R}$-valued random variable of the delay, which is a gap between a click and its resultant conversion. If $C=0$, it is not defined;
    \item $E$: $\mathbb{R}$-valued random variable of the elapsed time between a click and the training time 
\end{itemize}
,where $ \mathcal{X} \subset \mathbb{R}^d$.
In summary, the data structure is $(X, Y, E, D, C, S)$.
Because $C$ and $S$ are unobservable, the training data consist of samples $(x_i, y_i, e_i, d_i)$ where the lower case letter variables correspond to the realization of random variables. Note that when $y_i = 0$, then $d_i$ is empty.

The samples that are labeled as $Y=1$ in the training data are true positive ($C=1$). 
In other words, these samples are correctly labeled($S=1$) because $Y = C$.
Formally, $Y = 1 \Leftrightarrow S=1, C=1$.
In the delayed feedback, however, some positive samples ($C=1$) are mislabeled ($S=0$) when their elapsed time $E$ is shorter than their delay $D$.
Formally, $E < D \Rightarrow S = 0$. 
Hence, they are labeled as $Y=0$ although $C=1$.
For instance, the samples observed right before the end of the training dataset are yet to be converted but would eventually be.
Therefore, the negative samples in the training data consists of false and true ones.
Formally, $Y = 0 \Leftrightarrow C = 0 \text{  or  } S=0$. 

Based on the discussion above, the relation of the conditional distributions of $Y$ and $C$ are as follows:

\begin{align}
    \label{y_c_relation_1}
    P(Y = 1|X=x) &= P(C = 1|X=x)P(S=1| C = 1, X= x), \\
    \label{y_c_relation_0}
    P(Y = 0|X=x) &= P(C = 0|X=x) +  P(S = 0,C = 1 |X= x).
\end{align}
The first equation \eqref{y_c_relation_1} implies that the conditional probability $Y=1$ is equal to the conditional probability that a conversion occurs and is correctly labeled. 
The second equation \eqref{y_c_relation_0} implies that the conditional probability $Y=0$ is equal to the conditional probability that a conversion either does not occur or is mislabeled.

In the CVR prediction, our objective is to estimate $P(C=1|X)$ where it is impossible to observe $C$ in the training data. Therefore, we have to use $Y$ to train a CVR prediction model. 
We regard the delayed feedback as a problem in which there is a discrepancy between the conditional label distribution in the training $P(Y|X)$ and test $P(C|X)$ datasets although the feature distributions $P(X)$ remain the same. We refer to this situation as a feedback shift.

\subsection{Problem Formulation}
Let $L(x, y, \hat{y}): \mathcal{X} \times \mathcal{Y} \times \mathcal{Y} \rightarrow [0, \infty)$ be a loss function. For simplicity, we assume that a model is parametric. Let $\hat{f}(x, \theta)$ denote a model trained to predict $C$ where $\theta \in \Theta \subset \mathbb{R}^b$ is a parameter. 
The generalization error $G$, that is, the expected test error over the training samples, is denoted by \eqref{true_error}
\begin{equation}\label{true_error}
    G \equiv \mathbb{E}_{(x,c) \sim (X, C)} \Big[L \left(x, c; \hat{f}(x, \theta)\right)\Big].
\end{equation}
Let the optimal parameter $$\theta^* \in \argmin_{\theta \in \Theta} G.$$

Our aim is to estimate $\theta^{*}$ to obtain the CVR predictor. 
Typically, $\theta^*$ is estimated using the empirical risk minimization (ERM).

$c_i$ is required to calculate the empirical risk; however,  it is not available as it is not observed during the training period.
If $y_i$ is used instead, the empirical risk is defined as follows:

\begin{equation}\label{erm_error}
\hat{G}^{(n)} \equiv \frac1n\sum_{i=1}^n L \left(x_i, y_i; \hat{f}(x_i, \theta)\right).
\end{equation}
Minimizing the above empirical risk provides an estimator $\hat{\theta}_{ERM}$ that is consistent when $c_i$ and $y_i$ are extracted from the same conditional label distribution\cite{Shimodaira2000ImprovingPI}:

\begin{equation}\label{true_task}
    \hat{\theta}_{ERM}^{(n)} \equiv \argmin_{\theta \in \Theta} \hat{G}^{(n)}.
\end{equation}

In terms of the delayed feedback, there is a feedback shift that causes the discrepancy between the test distribution from which $c_i$ is drawn and the training distribution from which $y_i$ is drawn. 
Hence, the ERM estimator is not consistent for most cases. 
This is represented as follows:

$$ 
\lim_{n \rightarrow \infty} \hat{\theta}_{ERM}^{(n)} \neq \theta^*.
$$

Specifically, a CVR predictor would be prone to downward bias under the feedback shift because from \eqref{y_c_relation_1} and \eqref{y_c_relation_0} we have $P(Y = 1|X=x) \leq P(C = 1|X=x).$

\section{Importance Weight (IW) Approach}
To obtain the consistent ERM estimator, we introduce an  feedback shift importance weight (FSIW). 
First, we provide its theoretical background of FSIW, following which we propose a method for estimating it.

\subsection{Theoretical Background}
The loss weighted using the FSIW is defined below:
\begin{equation}\label{iwerm_error}
\hat{G}_{IW}^{(n)} \equiv \frac1n\sum_{i=1}^n \frac{P(C = y_i|X=x_i)}{P(Y = y_i|X=x_i)} L \left(x_i, y_i; \hat{f}(x_i, \theta)\right).
\end{equation}
Under the feedback shift, the loss weighted using the FSIW is consistent as follows.

\begin{theorem}
The loss weighted using the FSIW is consistent under the feedback shift, that is,
\begin{equation}
\lim_{n \rightarrow \infty} \hat{G}_{IW}^{(n)} = \mathbb{E}_{(x,c) \sim (X, C)} \Big[L \left(x, y; \hat{f}(x_i, \theta)\right)\Big].
\end{equation}
\end{theorem}

\begin{proof}
As the feature distribution does not change, and because of the law of large numbers, the following holds:
\begin{align}\label{consistency}
  \lim_{n \rightarrow \infty} \hat{G}_{IW}^{(n)} &= \mathbb{E}_{(x,y) \sim (X, Y)} \Bigg[\frac{P(C = y|X=x)}{P(Y= y|X=x)}L \left(x, y; \hat{f}(x, \theta)\right)\Bigg{]} \nonumber\\
  = \iint & \frac{P(C = y|X=x)}{P(Y=y|X=x)} L \left(x, y; \hat{f}(x, \theta)\right)P(X=x)P(Y=y|X=x)dydx \nonumber \\
  = \iint & {P(C=y|X=x)} L \left(x, y; \hat{f}(x, \theta)\right)P(X=x)dydx \nonumber \\
  = \iint & {L \left(x, c; \hat{f}(x, \theta)\right) P(C=c|X=x)} P(X=x)dcdx \nonumber \\
  =  &\mathbb{E}_{(x,c) \sim (X, C)} \Big[L \left(x, c; \hat{f}(x, \theta)\right)\Big]. \nonumber
\end{align}
\end{proof}

Note that this proof does not assume a specific model, a loss function, or any parameter learning method. 
Thus, the approach is valid for many models and algorithms.

Since \eqref{iwerm_error} is consistent with the true loss \eqref{true_error}, we can obtain a consistent estimator of $\theta^*$ by minimizing it. However, it would be impossible to directly estimate FSIW itself because its numerator is $P(C=y|X=x)$, which is what we finally wish to predict. 
Fortunately, in the delayed feedback case, we can indirectly estimate the FSIW. In the next section, we will explain how to estimate the FSIW in this situation.

\subsection{Estimation of FSIW}
From \eqref{y_c_relation_1} and \eqref{y_c_relation_0}, we obtain
\begin{eqnarray}\label{fsiw_separate}
    \frac{P(C = 1|X=x)}{P(Y = 1|X=x)} & = \frac{1}{P(S=1|C = 1,X= x)} \\ 
    \label{fsiw_separete2}
    \frac{P(C = 0|X=x)}{P(Y = 0|X=x)} & = 1 - \frac{P(S=0, C=1|X=x)}{P(Y=0|X=x)}
\end{eqnarray}
Therefore, instead of directly estimating the FSIW $\frac{P(C |X=x)}{P(Y|X=x)}$, we separately estimate the reciprocal of the probability of the occurrence of true positives $P(S=1|C = 1,X=x)$ and the probability of the occurrence of true negatives $(1 - \frac{P(S=0, C=1|X=x)}{P(Y=0|X=x)})$ using the elapsed time after a click as well as the other features.

To estimate these two probabilities, we prepare training data by artificially creating a situation in which delayed feedback occurs using the following steps.
First, we set a hypothetical deadline $\tau$ referred to as a \textit{counterfactual deadline}. 
Secondly, we discard samples that are clicked later than the counterfactual deadline. 
If the counterfactual deadline is sufficiently long, $Y=C$; thus we assume so.
In the Experiment Section, we analyze the performance of the proposed method in an instance where this assumption is violated.
Furthermore, we create a new elapsed time $e'_i$ between a click timestamp and the counterfactual deadline.
Thirdly, we label these remaining samples $S$. In the case where $Y=1$, we label $S$ according to whether the samples are converted before or after the counterfactual deadline. When $Y=0$, we set $S=1$.
Consequently, we obtain the \textit{artificial dataset} that has labels $S$.

In estimating $P(S=1|C=1,X=x)$, we only use the samples that are converted in the artificial dataset because $C=1$ is the same as $Y=1$ in this data.
After training a model $M$ using these samples, we make predictions of $S$ on the original training data with $Y=1$ including the data used to train the model to estimate $P(S=1|C=1, X=x)$.
\begin{algorithm}[t]
    \caption{Pseudo code of FSIW estimation}
    \label{iw}
    \begin{algorithmic}[1]
    \small
        \REQUIRE train data $\mathcal{D} = \{(x_i, y_i, e_i, d_i, ts_i)\}$, a counterfactual deadline $\tau$, model to estimate FSIW $M$, and $T$ is a timestamp when the data $\mathcal{D}$ are collected, where $x_i$ is a feature vector, $y_i$ is an observed label, $e_i$ is elapsed time since a click timestamp, $d_i$ is a delay which is a gap between a click timestamp and a CV timestamp, and $ts_i$ is a click timestamp.
        \ENSURE FSIW
        \STATE{$\mathcal{D}_{iw}^1, \mathcal{D}_{iw}^0$ = $\phi$}
        \FOR{$i = 1$ to {\em number of samples}}
            \IF{$ts_i < T - \tau$ and $y_i = 1$} 
                    \IF{$ts_i + d_i < T - \tau$} 
                        \STATE {Label the sample $i$ as $s_i=1$}
                    \ELSE
                        \STATE{Label the sample $i$ as $s_i = 0$}
                    \ENDIF 
                    \STATE{Insert the sample $(x_i, e_i - \tau, s_i)$ to $\mathcal{D}_{iw}^1$}
            \ENDIF
            \IF{$ts_i < T - \tau$ \AND ($ts_i + d_i >= T - \tau$ \OR $y_i$ =0)} 
                \IF{$y_i$ = 0} 
                    \STATE {Label the sample $i$ as $s_i=1$} 
                \ELSE 
                    \STATE{Label the sample $i$ as $s_i = 0$} 
                \ENDIF 
                \STATE{Insert the sample $(x_i, e_i - \tau, s_i)$ to $\mathcal{D}_{iw}^0$}
            \ENDIF
        \ENDFOR
        \STATE{Train $M$ on $\mathcal{D}_{iw}^1$ and predict $s$ on $\mathcal{D}$ with the label $y=1$}
        \STATE{Train $M$ on $\mathcal{D}_{iw}^0$ and predict $s$ on $\mathcal{D}$ with the label $y=0$}
        \RETURN{the reciprocal of the prediction for $y=1$ and the prediction for $y=0$, which are FSIW.}
    \end{algorithmic}
\end{algorithm}
To estimate $(1 - \frac{P(S=0, C=1|X=x)}{P(Y=0|X=x)})$, we use samples with $Y=0$ and those with $S=0$ in the artificial dataset. For the training and prediction process, the same procedure is conducted as the estimation of $P(S=1|C=1, X=x)$.
Note that the original elapsed time $e_i$ is used instead of $e'_i$ during prediction.
The detailed procedure is shown in Algorithm \ref{iw}.

\section{Experiment}
We conducted two different experiments. First, we evaluated our proposed method on the Criteo dataset\footnote{https://labs.criteo.com/2013/12/conversion-logs-dataset/} comparing it to the state-of-the-art method to show that the proposed method is more efficient.

Secondly, we incorporate the IW approach into the FFM (FFMIW) and evaluate the derived method, FFMIW on the offline in-house dataset that has three campaigns. 
Based on the offline evaluation, we conducted an A/B test on one of the campaigns to test the effectiveness of the proposed method.

\subsection{Public Dataset: Criteo Dataset}

\subsubsection{Dataset and Metrics}
We use the Criteo dataset used in \cite{chapelle2014modeling} to evaluate the proposed method. The experimental process is identical to that of \cite{chapelle2014modeling} to equally compare the method. 
The feature engineering, such as feature crossing and feature hashing is the same as that of \cite{chapelle2014modeling}. It also includes the timestamps of the clicks and those of the conversions, if any. Additionally, we include the elapsed time following a click, as a feature to estimate FSIW. We divide the original dataset into seven datasets as follows.
For each dataset, the training data starts at a specific time, and  
ends 3 weeks from that point. The test data begin immediately after the training data, ending in 1 day\footnote{The test data were tracked for 30 days; thus samples in the test data are tracked long enough to ensure that they are accurately observed.}.
For the next dataset, the training data start the day after the starting point of the previous training data, and so forth. 

The metrics used to evaluate our method are log loss (LL); area under the precision-recall curve (PR-AUC); and normalized log loss (NLL) that is the log loss normalized by that of the naive predictor that always predicts the average CVR of the training set. 
The PR-AUC is a more commonly used metric because it is more sensitive to skewed data than the AUC, and generally conversion log data are skewed\cite{Ktena2019AddressingDF}. 
In display advertising, the predicted probabilities are important because they are directly used to compute the value of an impression that is equal to a bid in an ad auction. 
Therefore, the LL and NLL are more important than PR-AUC. 
Moreover, the LL heavily depends on the mean of the label in the training dataset, that is, a CVR in the context of this setting. 
Hence, the NLL is a more effective metric for assessing the performance of the CVR predictors than the LL because the NLL is less sensitive to the background CVR\cite{He2014PracticalLF}. 

\subsubsection{FSIW Estimation and Hyperparameters}
In this experiment, we estimate the FSIW separately as mentioned in Section 5.2, using the LightGBM \cite{Ke2017LightGBMAH}.
We use the different hyperparamter settings for these two estimates. 
For estimating $P(S=1|C=1, X=x)$, the learning rate is 0.01, number of leaves is 64, and maximum depth is 6.
For estimating $1 - \frac{P(S=0, C=1|X=x)}{P(Y=0|X=x)}$, the learning rate is 0.01, number of leaves is 63, and maximum depth is 6.
The early stopping technique was applied for both settings to decide the number of trees. 
Note that a seven-day counterfactual deadline is set.
We use the estimated FSIW as a sample weight when training the model to predict the CVR. 

For learning a \textit{naive} logistic regression (LR) model that is only trained on the training data without taking the delay into consideration, we use the code provided with the Criteo dataset in \cite{chapelle2014modeling}. 
The reported LL in \cite{chapelle2014modeling} is reproduced when the L2 regularization parameter is 100. 
We applied this setting to all LR models to predict the CVR during this experiment to establish a fair basis of comparison of the methods.
\begin{center}
\begin{table}[h]
\small
\begin{tabular}{ c | c c c}
  & LR & DFM & LR-FSIW  \\ \hline
 LL & 0.4076 & 0.3989 & \textbf{0.3928*} \\ \hline
 PR-AUC & 0.6345 & 0.6481 & 0.6482 \\ \hline
 NLL & 25.21 & 27.33 & \textbf{28.02*}  \\ \hline
\end{tabular}
\caption{\label{tab:results}{Average metrics . LR: Logistic Regression without any consideration; DFM: Delayed Feedback Model proposed in \cite{chapelle2014modeling}; LR-FSIW: Logistic regression with FSIW. * means statistical significance in comparison to DFM.}}
\end{table}
\end{center}
\subsubsection{Result}
We compare our suggested model to the DFM\cite{chapelle2014modeling} and LR. 
We use LR equipped with the FSIW as our proposed model (LR-FSIW).
The result is provided in Table \ref{tab:results}. 
In the rest of the study, we used bootstrap \cite{Efron:2016:CAS:2994458} to calculate a 95\% confidence interval.
Our proposed method improves the LL by 1.5\% and the NLL by 2.5\% compared to the DFM and these improvements are statistically significant. 
However, there is no statistically significant difference in PR-AUC. 

To demonstrate the feasibility of our proposed method, we measured the training time of the two models for seven datasets. The time reported is the total amount of time. Whereas it takes approximately 140 h to train the DFM, the proposed method requires approximately 2.1 h\footnote{DFM is implemented in Cython. https://github.com/CyberAgent/delayedFeedback. We used N1 series with 8 CPU cores in Google Cloud Platform Compute Engine here.}.
\subsubsection{Counterfactual deadline}
In the Criteo dataset, an observational period which is a duration of tracking samples to fix their labels is 30 days. It would be natural to set a 30-day counterfactual deadline. However, the deadline is so long that the data to be used to estimate the FSIW become obsolete. 
Hence, it is necessary to shorten the counterfactual deadline, however, there is no means of setting the counterfactual deadline in a natural manner. 
Therefore, we vary the counterfactual deadline from one to seven days to evaluate the performance. The result, as shown in Figure \ref{fig:cf_nll}, confirms the stability of our proposed method.
\begin{figure}[h]
\includegraphics[width=5cm, height=2.8cm]{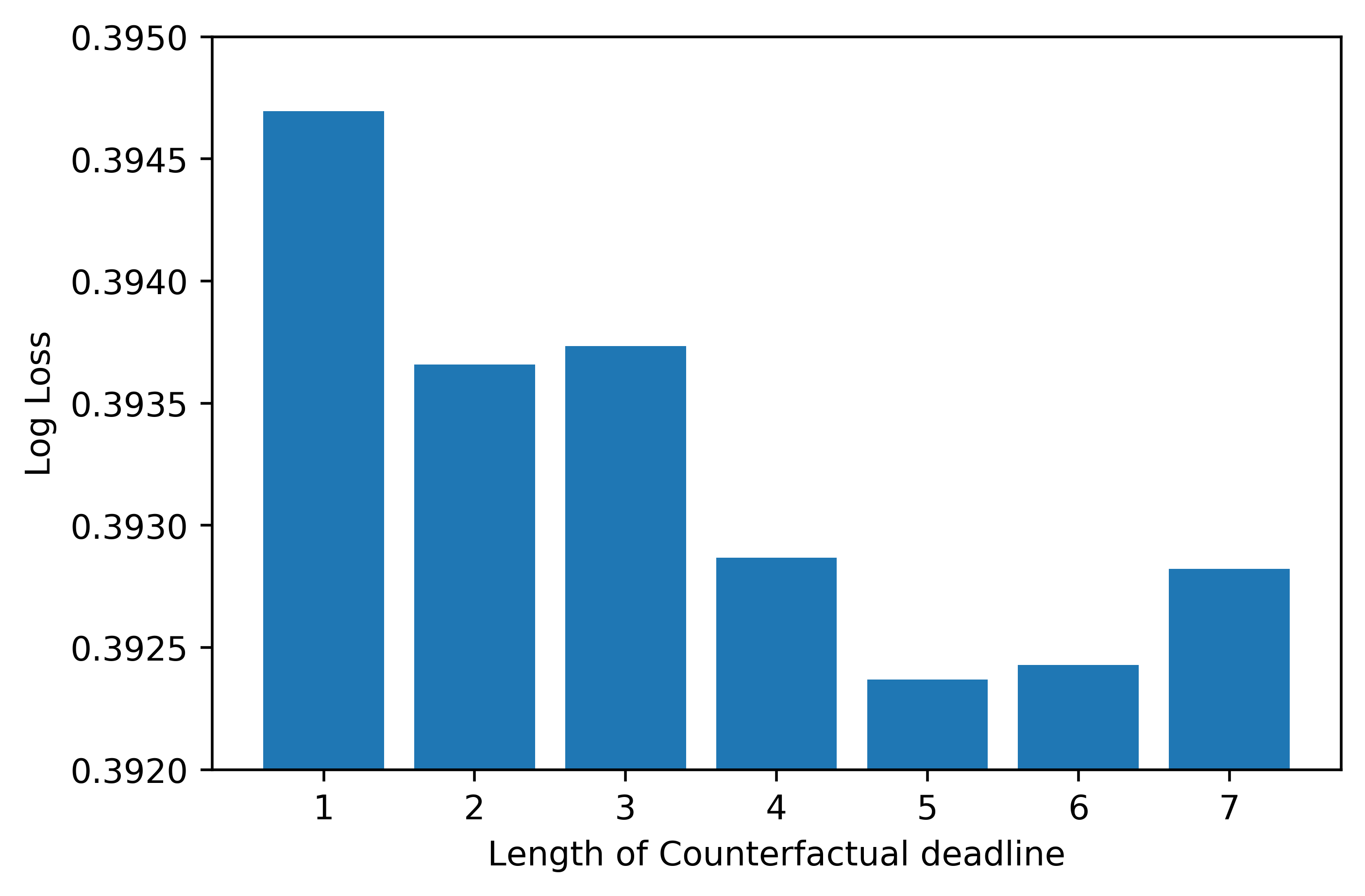}
\caption{LL of different counterfactual deadline lengths}
\label{fig:cf_nll}
\end{figure}
\subsection{Dataset: Dynalyst Dataset}
\subsubsection{Dataset and Metrics}
The experiments are performed using the in-house data provided by Dynalyst. The dataset has three campaigns, the candidates on which our proposed method would be tested online.

In our production environment, the observational period of each campaign is different because it is decided by different advertisers. 
For instance, Campaign L has a 1-week-long observational period, Campaign M has a 3-day-long observational period, and Campaign S has a 1-day-long observational period.
The data has eight categorical features and five numerical features that are categorized. 
Although we use the hashing trick, the dimensionality is $10^5$. The datasets are divided into 16 sets, as in the previous experiment. 
The difference is that there is a 1-day validation set between the training and test data. Each dataset has a 13-day training set, 1-day validation set, and 1-day test set. 
\subsubsection{FSIW Estimation and Hyperparameters}
In our production environment, we create a model for each campaign. Thus, we trained three models and evaluated them for each campaign.

The hyperparameters setting used to estimate the FSIW are the same, as in the previous experiment. 
In this experiment, the observational period is designated the counterfactual deadline. For example, the observational period in Campaign L is seven days. 
Therefore, a seven-day counterfactual deadline is set. 
Because samples collected 7 days before are correctly labeled, it is reasonable for the counterfactual deadline to be set in such a manner.

For learning a CVR predictor, we used the FFM \cite{Juan2016FieldawareFM}. The hyperparameters are as follows. There are four latent factors, and the regularization parameter is 0.00002; these settings are identical for the FFM and FFMIW. 
In addition, we decide the number of iterations by applying early stopping using the validation set; the validation set is weighted according to the FSIW when we train FFMIW.
Finally, we combined the training and the validation set, following which trained the model.
\subsubsection{Result}
We compare the FFM and FFMIW. 
We evaluated 16 sets, and report the average metrics in Table \ref{tab:ffm_results} for three different campaigns.
In all the campaigns, the FFMIW seems to be better than the FFM. 
However, only the difference in the NLL of Campaign L is statistically significant. 
As discussed in Section 6.1.1, the NLL is the most important metric in the production environment; thus we concluded that the FFMIW outperformed the FFM in Campaign L only. 

This is because Campaign S and M have relatively shorter observational period, which makes the delay less influential. 
\begin{table}[]
\small
\begin{tabular}{l|l|l|l|l}
                            &       & LL              & PR-AUC & NLL             \\ \hline
\multirow{2}{*}{Campaign L} & FFM   & 0.3523          & 0.1612 & 1.7197 \\ \cline{2-5} 
                            & FFMIW & \textbf{0.3500} & \textbf{0.1660} & \textbf{2.304*}  \\ \hline
\multirow{2}{*}{Campaign M} & FFM   & 0.2409 & 0.0808 &  0.2160 \\ \cline{2-5} 
                            & FFMIW & \textbf{0.2401}  & \textbf{0.0828} & \textbf{0.3771}          \\ \hline
\multirow{2}{*}{Campaign S} & FFM   & 0.4026          & 0.2055 & 2.9953          \\ \cline{2-5} 
                            & FFMIW & \textbf{0.3967} & \textbf{0.2058} & \textbf{3.361} \\
\end{tabular}
\caption{\label{tab:ffm_results}{Average metrics of FFM and FFMIW for 3 different campaigns. * denotes statistical significance.}}
\end{table}
\subsection{Online Experiment}
According to our offline experiment results, the proposed method outperforms the existing method in Campaign L; thus we conducted 14 days of A/B testing for the FFM and FFMIW in Campaign L.

We trained the FFM and FFMIW once a day with a 14-day-long training data. 
The number of iterations and the hyperparameters are decided in the same way as described in Section 6.2.
During the A/B test period, we equally randomly selected one of the two models to predict the CVR every time a bid request arrives. 
Note that this A/B test is applied to approximately one million advertising impressions in this period.

The results are listed in Table \ref{tab:ab_results}. 
We observed that there is a statistically significant increase in the number of conversions(CV) and the consumed costs compared to the FFM. 
This is because since the FFM ignores the delayed feedback and its predicted values are thus subject to downward bias. 
However, the FFMIW considers the delay, and thus, places relatively higher bids. This results in the higher costs and the acquisition of more CV. 
The CPA was 2\% lower. However, it was not a statistically significant difference. These results indicate that the FFMIW incurred more costs, and performed similarly to the FFM in obtaining a conversion. 
For the CPA model, a higher cost is suitable if a CPA is the same or lower. 
This is because advertisers are able to conduct a larger campaign, and the total sales of the product increases. 
Therefore, the FFMIW outperformed the FFM in our production environment.
\begin{table}[ht]
\centering
\begin{tabular}{ccc}
  \hline
 CV & Cost & CPA \\ 
  \hline
 +31\%* & +28\%* & -2\% \\ 
   \hline
\end{tabular}
\caption{\label{tab:ab_results}{Online relative comparison of FFM and FFMIW. The values shown are the relative change in the FFMIW against the FFM. * denotes statistical significance.}}
\end{table}
\section{Conclusion}
In this study, we tackle on the delayed feedback in CVR prediction by using the IW technique, which constructs the consistent empirical loss.
We empirically show that our proposed method performs better than the existing methods in the Criteo dataset.
Furthermore, we incorporate the IW approach into the FFM, and compare its performance to that of the FFM without the IW using the in-house dataset.
Finally, we conducted an online A/B test to confirm the effectiveness of our proposed method in the production system.

\bibliographystyle{ACM-Reference-Format}
\bibliography{atmf}

\end{document}